\documentclass{article}

\usepackage{microtype}
\usepackage{graphicx}
\usepackage{subcaption}
\usepackage{booktabs} %
\usepackage{multirow}
\usepackage{enumitem}
\usepackage{minitoc}

\setlist[itemize]{leftmargin=0.5cm, itemsep=1pt, parsep=0pt}

\usepackage{hyperref}
\usepackage{tcolorbox}

\usepackage[preprint]{icml2026}

\usepackage{amsmath}
\usepackage{amssymb}
\usepackage{mathtools}
\usepackage{amsthm}

\usepackage{colortbl}
\definecolor{mygray}{gray}{.9}

\usepackage[capitalize,noabbrev]{cleveref}

\usepackage{xcolor}
\definecolor{mydarkred}{rgb}{0.6,0,0}
\definecolor{mydarkgreen}{rgb}{0,0.6,0}

\hypersetup{
    colorlinks=true,
    linkcolor=mydarkred,
    citecolor=mydarkgreen
}

\theoremstyle{plain}
\newtheorem{theorem}{Theorem}[section]
\newtheorem{proposition}[theorem]{Proposition}

\theoremstyle{definition}
\newtheorem{definition}[theorem]{Definition}

\theoremstyle{remark}

\usepackage[textsize=tiny]{todonotes}

\newcommand{\methodname}{RLPT}

\icmltitlerunning{{Reinforcement Learning with Promising Tokens for Large Language Models}}

\begin{document}

\doparttoc
\faketableofcontents

\twocolumn[
  \icmltitle{Reinforcement Learning with Promising Tokens for Large Language Models}

  \icmlsetsymbol{equal}{*}

  \begin{icmlauthorlist}
    {\large A} {\large P}\normalsize REPRINT \\ \vspace{1em}
    \icmlauthor{Jing-Cheng Pang}{huawei}
    \icmlauthor{Liang Lu}{huawei}
    \icmlauthor{Xian Tang}{huawei}
    \icmlauthor{Kun Jiang}{huawei}
    \icmlauthor{Sijie Wu}{huawei}
    \icmlauthor{Kai Zhang}{huawei}
    \icmlauthor{Xubin Li}{huawei}
  \end{icmlauthorlist}
  \icmlaffiliation{huawei}{Huawei Technologies Co., Ltd.}
  \icmlcorrespondingauthor{Jing-Cheng Pang}{pangjcheng@gmail.com}

  \vskip 0.3in
]

\printAffiliationsAndNotice{}  %

\begin{abstract}
Reinforcement learning (RL) has emerged as a key paradigm for aligning and optimizing large language models (LLMs). Standard approaches treat the LLM as the policy and apply RL directly over the full vocabulary space. However, this formulation includes the massive tail of contextually irrelevant tokens in the action space, which could distract the policy from focusing on decision-making among the truly reasonable tokens. In this work, we verify that valid reasoning paths could inherently concentrate within a low-rank subspace. Based on this insight, we introduce Reinforcement Learning with Promising Tokens (\methodname), a framework that mitigates the action space issue by decoupling strategic decision-making from token generation. Specifically, \methodname~leverages the semantic priors of the base model to identify a dynamic set of \emph{promising tokens} and constrains policy optimization exclusively to this refined subset via masking. Theoretical analysis and empirical results demonstrate that \methodname~effectively reduces gradient variance, stabilizes the training process, and improves sample efficiency. Experiment results on math, coding, and telecom reasoning show that \methodname~outperforms standard RL baselines and integrates effectively across various model sizes (4B and 8B) and RL algorithms (GRPO and DAPO).
\end{abstract}

\section{Introduction}
\label{sec:intro}

Developing large language models (LLMs) capable of learning and improving autonomously is a hallmark of machine intelligence. Reinforcement Learning (RL) has emerged as the primary paradigm for this objective, enabling LLMs to improve by optimizing against explicit reward signals that measure the quality of the model's response \citep{rlc,rule_reward}. By formulating the LLM as a policy, RL enables the maximization of long-term rewards, clearly enhancing performance in tasks requiring complex reasoning, such as math and coding \citep{grpo}.

\begin{figure}[t]
 \centering
 \includegraphics[width=\linewidth]{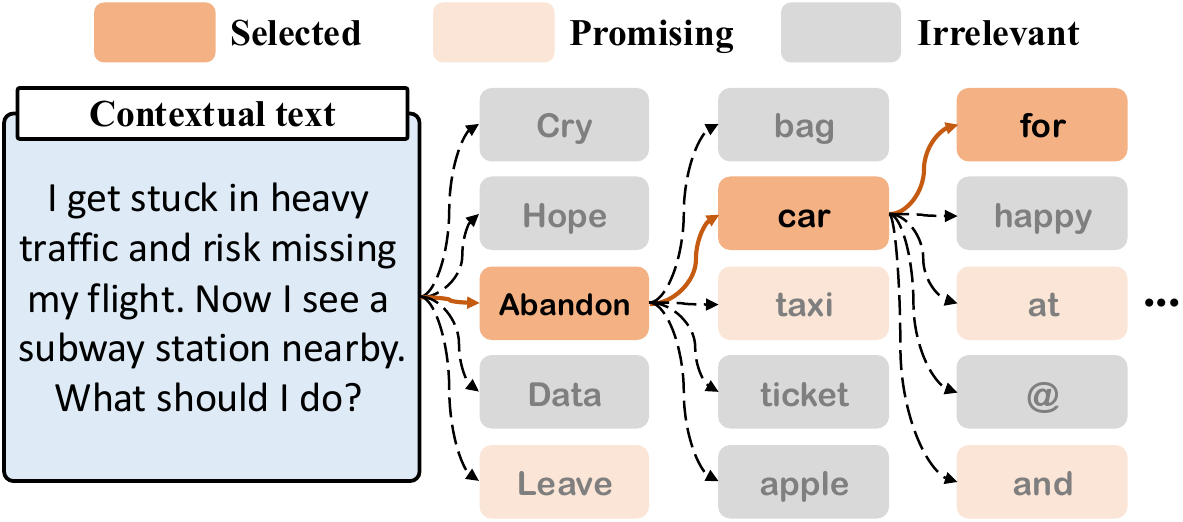}
 \caption{An illustration of decision making with promising tokens. At each step, the policy selects tokens solely from a high-likelihood subset, enabling it to focus on strategic decision-making.}
 \vspace{-2em}
 \label{fig:illustration}
\end{figure}

However, directly applying standard RL algorithms (e.g., PPO \citep{ppo}, REINFORCE \citep{reinforce}) to LLMs presents a unique challenge regarding the high dimensionality of the action space \citep{llm_need_small_actionspace}. Unlike traditional RL tasks with compact action spaces (e.g., Atari, Robot control  \citep{atari,mujoco}), the action space of an LLM corresponds to its full vocabulary, typically exceeding 50,000 tokens.
In practice, previous works have attempted to alleviate this issue using inference strategies such as Top-$k$ \citep{topk} or Nucleus sampling \citep{topp}, which prune the token space while maintaining text coherence \textit{during the rollout phase} \citep{verl}. Yet, the policy optimization step still operates over the vocabulary size, creating a off-policy mismatch between efficient rollout and high-dimensional policy optimization. Consequently, the RL signal is distracted by the need to maintain basic syntactic coherence rather than focusing solely on the logical reasoning required for improvement.

To bridge this gap, we suggest that RL for LLM policies should operate within a more focused decision space, as illustrated in Fig. \ref{fig:illustration}. We begin by investigating the distribution of successful trajectories and empirically verify that the tokens required for correct solutions are naturally concentrated within a small subspace of high-likelihood candidates. Building on this insight, we propose Reinforcement Learning with Promising Tokens (\methodname) method, which \textit{decouples logical decision-making from token generation}. 
The core intuition of \methodname~is to enable the LLM policy to make decisions over the subspace of the token space, and focus more on strategic selection. To achieve this, \methodname~leverages the semantic priors of the pre-trained base model to identify a dynamic set of \emph{promising tokens}, i.e., candidates with high likelihood that ensure syntactic correctness. At each generation step, \methodname~constructs a refined action space that restricts the RL process to optimize solely over these promising tokens via masking. By filtering out the long tail tokens, \methodname~shifts the burden of the policy from \textit{predicting the next word} to \textit{selecting the best move}, allowing the agent to prioritize logical reasoning over the syntactic maintenance.

Our contributions are summarized as follows. We introduce the concept of \emph{learning with promising tokens}, a paradigm that restricts policy optimization to a semantically valid subspace to improve learning efficiency. Our empirical analysis provides a strong empirical foundation for pruning the action space. Besides, we propose \methodname, a practical framework that implements this concept by dynamically masking irrelevant tokens during both the rollout and training phases. Furthermore, we provide a theoretical analysis demonstrating that optimizing over the constrained promising set significantly reduces the variance of the policy gradient estimator, thereby stabilizing training. Finally, extensive experiments on mathematical reasoning and coding tasks demonstrate that \methodname~achieves superior performance and sample efficiency compared to strong RL baselines, including GRPO \citep{grpo} and DAPO \citep{dapo}.

\section{Preliminary}
\label{sec:preliminary}

\textbf{RL for LLM optimization}. 
We consider a standard language modeling task in which an LLM serves as a policy $\pi_\theta$ parameterized by $\theta$. Given a context input (prompt or question) $x$, the model generates a response $y$ composed of a sequence of tokens $y = (y_1, y_2, \dots, y_T)$, where each token $y_t$ belongs to a fixed vocabulary space $\mathcal{V}$. The probability of generating the entire sequence is given by the chain rule: $\pi_\theta(y \mid x) = \prod_{t=1}^{T} \pi_\theta(y_t \mid x, y_{<t}),$ where $y_{<t}$ denotes the partial sequence generated prior to step $t$. 

To align the LLM with human preferences or specific logical rules, we formulate the generation process as a Markov Decision Process (MDP \citep{puterman2014markov}) defined by the tuple $\langle \mathcal{S}, \mathcal{A}, \mathcal{T}, \mathcal{R}, \gamma \rangle$:
\begin{itemize}
    \item State space $\mathcal{S}$: The state $s_t$ at time step $t$ consists of the prompt and the history of generated tokens, i.e., $s_t = (x, y_1, \dots, y_{t-1})$;
    \item Action space $\mathcal{A}$: the space of the tokens $\mathcal{V}$;
    \item Transition $\mathcal{T}$: $s_{t+1} = s_t \cup \{a_t\}$;
    \item Reward function $\mathcal{R}$: a score $r(s_t, a_t)$ that reflects the quality of the generated answer to the question, typically given at the end of the sequence (e.g., $r_T = R(x, y)$) based on correctness verification, with intermediate rewards being zero;
    \item Discount factor $\gamma$: We typically set $\gamma=1$ for episodic generation tasks.
\end{itemize}
The goal of RL training is to find an optimal policy $\pi_\theta$ that maximizes the expected cumulative reward over the prompt distribution $\mathcal{D}$:
$$J(\theta) = \mathbb{E}_{x \sim \mathcal{D}, y \sim \pi_\theta(\cdot|x)} [R(x, y)],$$
where $R(x, y)$ represents the reward (e.g., answer correctness) received at the end of the episode.
\section{Analyzing the Potential for Decision Making over Promising Tokens}
\label{sec:empirical_analysis}

Before introducing our method, we answer an important question: \textit{Is the subspace of promising tokens sufficient for LLMs to solve the problem?}  
Intuitively, while the vocabulary of an LLM is vast ($|\mathcal{V}|> 50k$), the tokens required for a logically correct step are often concentrated within a small subset of high-probability candidates. This is supported by the successful application of top-k sampling trick.
In this section, we conduct an empirical analysis using the Qwen series models \citep{qwen} to answer this question. We first formally define the concept of promising tokens.
\begin{definition}[\textbf{Promising Tokens}]
\label{def:promising_tokens}
\emph{
Given the current state $s_t$ and the policy distribution $\pi(\cdot|s_t)$, let $\text{rank}(v \mid s_t)$ denote the rank of a token $v \in \mathcal{V}$ when sorted in descending order of its probability. The set of promising tokens, denoted as $\mathcal{P}_t$, is defined as the top-$K$ candidates:
\begin{equation}
    \mathcal{P}_t = \left\{ v \in \mathcal{V} \mid \text{rank}(v \mid s_t) \le N \right\}.
\end{equation}
Consequently, $|\mathcal{P}_t| = K$ and $\mathcal{P}_t \subset \mathcal{V}$.}
\end{definition}
This definition constructs a focused action space, filtering out the long tail of contextually irrelevant tokens. To verify the coverage of $\mathcal{P}_t$, we analyze the token ranks of \textit{successful trajectories} from two distinct perspectives: the ground truth answer and the model's own successful answer.

\textbf{Setup.} We employ \texttt{Qwen3-8B} and \texttt{Qwen3-32B} as base models. We use three diverse datasets: GSM8K \citep{gsm8k} for mathematics, HumanEval \citep{humaneval} for coding, and AlpacaEval \citep{alpaca} for general instruction following.
For each step $t$ in a successful trajectory $y$, we compute the rank of the target token $y_t$ within the base model's predicted distribution $\pi_{\text{base}}(\cdot | y_{<t})$. We define the \textit{top-$K$ coverage rate} as the percentage of tokens in these trajectories that satisfy $y_t \in \mathcal{P}_t$.

\begin{table}[htbp]
    \centering
    \caption{Top-$K$ coverage rate (\%) of successful trajectory tokens from labeled solutions. Math, Code, and General denote GSM8K, HumanEval, and AlpacaEval dataset, respectively.}
    \label{tab:token_rank_coverage_labeled_answer}
    \resizebox{\linewidth}{!}{
    \begin{tabular}{l|ccc|ccc}
        \toprule
        \multirow{2}{*}{\textbf{Metric}} & \multicolumn{3}{c|}{\textbf{Qwen3-8B}} & \multicolumn{3}{c}{\textbf{Qwen3-32B}} \\
        & Math & Code & General & Math & Code & General \\
        \midrule
        Top-2  & 91.5 & 92.0 & 83.0 & 92.1 & 92.7 & 83.8 \\
        Top-4  & 95.3 & 95.0 & 91.0 & 96.1 & 95.3 & 91.7 \\
        Top-8 & 97.4 & 96.5 & 95.4 & 98.1 & 96.7 & 95.9 \\
        Top-16 & 98.6 & 97.2 & 97.6 & 99.0 & 97.5 & 98.0 \\
        Top-32 & 99.2 & 97.7 & 98.8 & 99.5 & 98.0 & 99.0 \\
        \bottomrule
    \end{tabular}
    }
    \vspace{-1.5em}
\end{table}

\subsection{Analysis of Ground-truth Solution}

\textbf{\colorbox{red!15}{Observation 1:} \emph{The labeled solution is effectively contained within the promising token subspace.}}

We first investigate whether the ground truth answers (labeled solutions) fall within the promising tokens predicted by the base model. This measures the \textit{feasibility} of finding the optimal policy within the promising token space $\mathcal{P}_t$.  As shown in Tab. \ref{tab:token_rank_coverage_labeled_answer}, the correct answer tokens are highly concentrated in the head of the distribution. For \texttt{Qwen3-32B}, over $99.5\%$ of the ground truth tokens fall within the Top-32 candidates for Math tasks. Even for the smaller 8B model, the Top-32 coverage consistently exceeds $97\%$ across all domains.
Besides, qualitative inspection reveals that the few outliers (i.e., tokens outside $\mathcal{P}_t$) mainly appear at the very beginning of the answer generation or consist of infrequent symbols, \textbf{typically having negligible impact on the core reasoning logic.}
This observation confirms that the optimal moves required to solve the problem are already statistically prioritized by the pre-trained model. Therefore, pruning the vocabulary to $\mathcal{P}_t$ does not impose an upper bound on performance, as the solution space remains intact.

\subsection{Analysis of Model-generated Solution}

\textbf{\colorbox{red!15}{Observation 2:} \emph{The model's intrinsic answering capabilities do not rely much on the long tail of the vocabulary.}}

While previous experiment confirms that the labeled solution is largely covered by promising tokens, one might worry that the model needs to explore creative low-probability tokens to find its own path to the solution. To address this, we sample multiple answers from the base model and analyze only those that reach the correct final answer.
Tab. \ref{tab:token_rank_coverage_model} reports the rank of tokens in these \textit{model-generated successful paths}. The results show that for Math and Code tasks, $100\%$ of the tokens in correct solutions are located within the Top-8 candidates.
This indicates that when the model successfully reasons, it rely most on the high-likelihood tokens. The long tail tokens contribute virtually nothing to correct reasoning chains. Consequently, the challenge of RL may not be expanding the search to the full vocabulary, but to learn to decide the correct move from the promising tokens.

\begin{table}[htbp]
    \centering
    \caption{Top-$K$ coverage rate (\%) of successful trajectory tokens from model-generated solutions.}
    \label{tab:token_rank_coverage_model}
    \resizebox{\linewidth}{!}{
    \begin{tabular}{l|ccc|ccc}
        \toprule
        \multirow{2}{*}{\textbf{Metric}} & \multicolumn{3}{c|}{\textbf{Qwen3-8B}} & \multicolumn{3}{c}{\textbf{Qwen3-32B}} \\
        & Math & Code & General & Math & Code & General \\
        \midrule
        Top-2  & 98.9 & 98.1 & 93.2 & 92.1 & 97.8 & 92.4 \\
        Top-4  & 99.9 & 99.6 & 98.0 & 100 & 100 & 100 \\
        Top-8  & 100 & 100 & 100 & 100 & 100 & 100 \\
        \bottomrule
    \end{tabular}
    }
\end{table}

\section{Method}
\label{sec:method}

\begin{figure*}[t]
\centering
\includegraphics[width=0.95\linewidth]{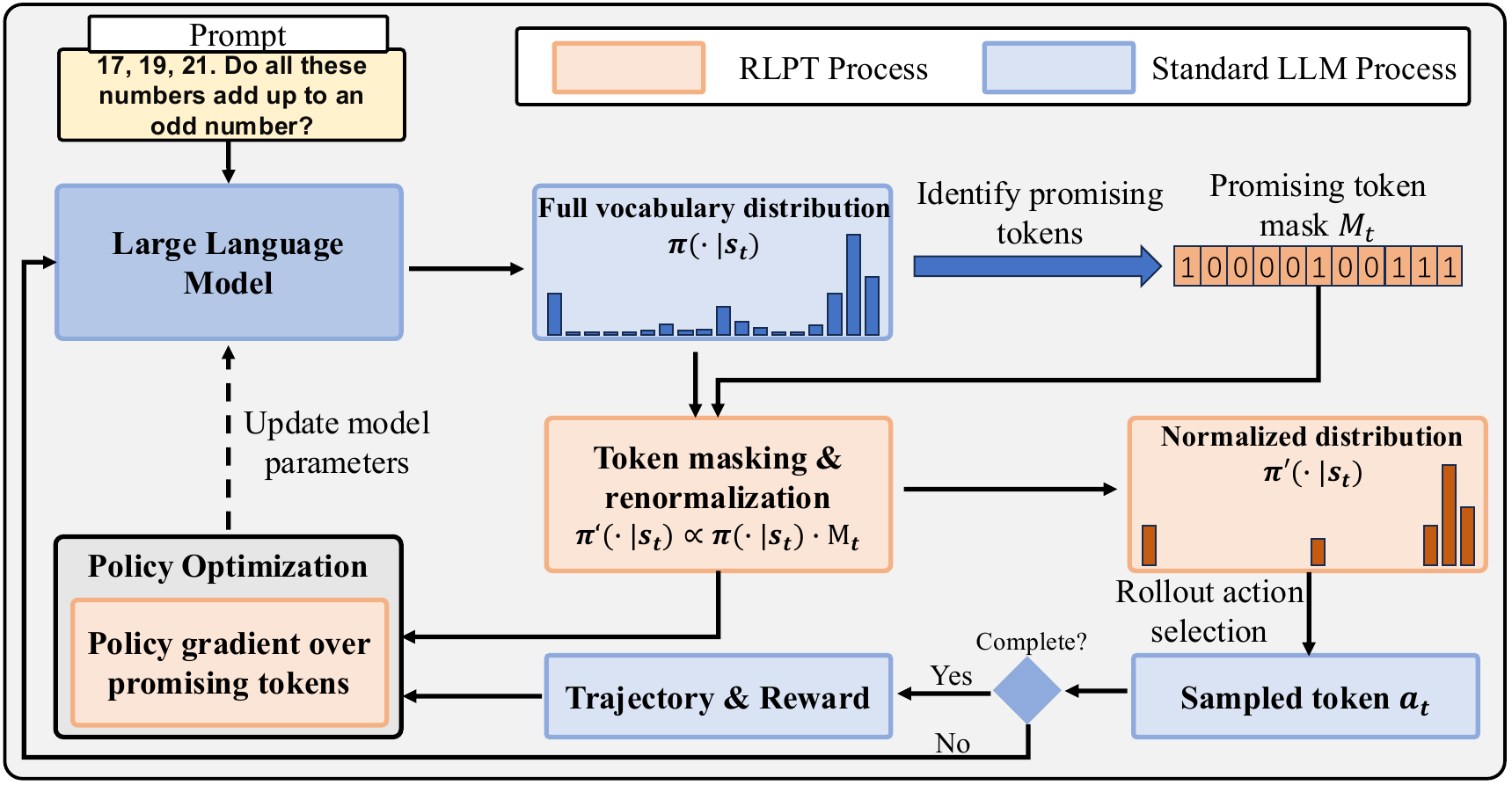}
\caption{The overall framework of \methodname~method.}
\vspace{-1em}
\label{fig:overall_framework}
\end{figure*}

In this section, we present \methodname, a framework designed to improve RL optimization for LLMs by constraining the optimization landscape to a semantically valid subspace, as shown in the overall framework in Fig. \ref{fig:overall_framework}. The core mechanism of \methodname~operates in a two-stage process: first, constructing a binary mask $M_t$ based on the semantic priors of the behavior policy; second, applying this mask to decouple the logical decision-making from the massive vocabulary space during both the sampling and training phases.

\subsection{Promising Token Construction}
\label{sec:mask_construction}

The foundation of \methodname~is the identification of \emph{promising tokens}. At each time step $t$, given the current state $s_t$, we utilize the behavior policy $\pi_{\text{old}}$ (the policy used for data collection) to compute the probability distribution over the vocabulary $\mathcal{V}$.

We formally define the \textit{Promising Mask} $M_t \in {0, 1}^{|\mathcal{V}|}$ as a binary vector that filters out contextually irrelevant tokens. Let $\mathcal{P}_t$ be the set of top-$K$ tokens determined by $\pi_{\text{old}}(\cdot|s_t)$. The elements of the mask vector $M_t$ are defined as:
\begin{equation}
M_t[v] = \mathbb{I}(v \in \mathcal{P}_t) =
\begin{cases}
1, & \text{if } v \in \mathcal{P}_t \\
0, & \text{otherwise}
\end{cases},
\end{equation}
where $\mathbb{I}(\cdot)$ is the indicator function. This mask filter for the generation ensures that subsequent operations focus exclusively on candidates that are syntactically and semantically plausible according to the base model’s priors.

\subsection{Policy Rollout and Optimization with Promising Token Masking}

\methodname~integrates the promising mask into the standard RL pipeline by modifying both the rollout (sampling) and the policy update (training) processes.

\paragraph{Policy Rollout.}
During the data generation phase, we aim to explore diverse reasoning paths while maintaining coherence. Instead of sampling from the full distribution, we sample from a masked distribution $\tilde{\pi}$. Specifically, the probability of selecting a token $a_t$ is renormalized over the promising set:
\begin{equation}
\tilde{\pi}_{\text{old}}(a_t \mid s_t) = \frac{\pi_{\text{old}}(a_t \mid s_t) \cdot M_t[a_t]}{\sum_{v \in \mathcal{V}} \pi_{\text{old}}(v \mid s_t) \cdot M_t[v]}.
\end{equation}
This operation is computationally equivalent to performing Top-$K$ sampling. Crucially, the mask vector $M_t$ alongside the trajectory could be stored for the training phase.

\paragraph{Policy Optimization.}
A key innovation of \methodname~is consistent masking during optimization. Standard RL maximizes the likelihood of actions across the entire vocabulary, introducing noise from the long tail of irrelevant logits. In contrast, \methodname~optimizes the policy $\pi_\theta$ to select the best token \textit{within the promising subspace}.

During the backward pass, we apply the pre-computed mask $M_t$ (derived from the behavior policy) to the current policy’s logits. The gradient updates are calculated using the masked probability:
\begin{equation}
\tilde{\pi}_\theta(a_t \mid s_t) = \text{Softmax}\left( \text{Logits}_\theta(s_t) + (1 - M_t) \cdot (-\infty) \right)_{a_t}.
\end{equation}
By forcing the probability mass of masked tokens to zero, we ensure that the policy is not penalized for fluctuations in the logits of irrelevant tokens.

\subsection{Integration with RL Algorithms}
\methodname~is algorithm-agnostic and can be seamlessly integrated with arbitrary RL objectives. We illustrate this using Group Relative Policy Optimization (GRPO) \citep{grpo} as a representative baseline.
Given a group of trajectories ${Traj_1, \dots, Traj_G}$ sampled with the masked policy, the objective function of \methodname-GRPO is formulated as:
\begin{equation}
\small
\begin{split}
\mathcal{J}_{\text{\methodname}}(\theta) = &\mathbb{E}_{Traj \sim \tilde{\pi}_{\text{old}}} \Bigg[ \frac{1}{T} \sum_{t=1}^T \min \Bigg(  \frac{\tilde{\pi}_\theta(a_t|s_t)}{\tilde{\pi}_{\text{old}}(a_t|s_t)} A_t, \\
& \text{clip}\left(\frac{\tilde{\pi}_\theta(a_t|s_t)}{\tilde{\pi}_{\text{old}}(a_t|s_t)}, 1-\epsilon, 1+\epsilon\right) A_t \Bigg) \Bigg],
\end{split}
\end{equation}
where $A_t$ is the advantage computed from rewards. The key difference from standard GRPO is that the probability ratio $\frac{\tilde{\pi}_\theta}{\tilde{\pi}_{\text{old}}}$ is computed strictly within the promising subspace defined by $M_t$. This ensures that the policy improvement step is perfectly aligned with the exploration step.

\subsection{Theoretical Justification}

In this subsection, we provide a theoretical analysis to justify why optimizing over a constrained promising set $\mathcal{P}_t$ leads to more stable training compared to the full vocabulary space $\mathcal{V}$. We focus on the variance of the policy gradient estimator, which is a key factor in the convergence speed and stability of RL algorithms.

Consider the gradient of the objective function $J(\theta)$ with respect to the logits $z_t \in \mathbb{R}^{|\mathcal{V}|}$ of the policy at step $t$. The standard policy gradient estimator $\hat{g}$ can be expressed as:
\begin{equation}
\hat{g} = \nabla_{z_t} \log \pi_\theta(a_t|s_t) \cdot A_t,
\end{equation}
where $A_t$ is the advantage function. For the softmax parameterization $\pi(a|s) = \frac{e^{z_a}}{\sum_{v \in \mathcal{V}} e^{z_v}}$, the gradient of the log-probability with respect to the logit vector $z$ is given by $\nabla_z \log \pi(a|s) = \mathbf{e}_a - \pi$, where $\mathbf{e}_a$ is the one-hot vector for action $a$ and $\pi$ is the probability vector over $\mathcal{V}$.

Let $\mathcal{T} = \mathcal{V} \setminus \mathcal{P}_t$ denote the set of ``tail’’ tokens that are masked out in \methodname. We analyze the variance of the gradient estimator by decomposing the contribution of the promising set $\mathcal{P}_t$ and the tail set $\mathcal{T}$.

\begin{proposition}[\textbf{Variance Reduction}]
\label{prop:variance}
Assuming the advantage $A_t$ is bounded, optimizing the policy over the constrained space $\mathcal{P}_t$ strictly reduces the variance of the gradient estimator associated with the tail tokens $\mathcal{T}$, compared to optimization over the full vocabulary $\mathcal{V}$.
\end{proposition}

We present the proof of Prop. \ref{prop:variance} in Appendix \ref{appsec:proof}.

The theoretical analysis guarantees that \methodname~strictly reduces the variance of the gradient estimator by eliminating the stochastic fluctuations arising from the high-dimensional tail, providing a stable optimization landscape for efficient policy learning.
\section{Related Work}
\label{sec:related_work}

In this section, we review related work from the following three areas.

\subsection{Reinforcement Learning for LLM Training}
RL has became the standard method for aligning LLMs with human intent and enhancing their reasoning capabilities \citep{edco,deepseekv3,kalm}. Early works, such as RLHF \citep{rlhf}, employ the PPO algorithm \citep{ppo} to optimize policies with respect to learned reward models. DPO \citep{dpo} and its variants \citep{ipo,kto} simplify this process by implicitly optimizing the reward function, yet they fundamentally rely on the same policy gradient formulation. Beyond general alignment, RL has shown remarkable efficacy in reasoning-intensive domains. Methods like GRPO \citep{grpo} and RFT \citep{rft} leverage outcome-based supervision to improve mathematical and code reasoning. However, a common limitation across these approaches is that they typically formulate the policy optimization problem over the entire vocabulary space. This high-dimensional action space introduces gradient variance and potentially explores a vast number of contextually irrelevant tokens, making the training process sample-inefficient and unstable. This work addresses this fundamental inefficiency by redefining the optimization landscape from the full vocabulary to a refined subspace of promising candidates.

\subsection{Decoding-time Action Pruning Strategies}

Managing the vast action space of LLMs is a widely studied problem during LLM inference. Deterministic strategies, such as greedy decoding, often yield repetitive loops \citep{llm_repeat}, while stochastic methods, such as Top-$k$ sampling \citep{topk} and nucleus sampling \citep{topp}, are widely used to truncate the tail of the probability distribution. These methods effectively prune implausible tokens to ensure linguistic coherence and diversity during text generation. However, these pruning strategies are primarily used as \textit{decoding heuristics} during the rollout phase. The subsequent policy optimization step in standard RL algorithms \citep{verl,trlx} typically ignores this structure, either by calculating gradients based on the full policy distribution or by failing to theoretically align the training objective with the pruned search space. This discrepancy creates a \textit{train-inference mismatch}: the model explores within a constrained subspace but updates its parameters as if it had a global action space. We discuss this mismatch in Appendix \ref{appsec:discussion}.
Unlike these heuristic approaches, \methodname~integrates token pruning into the RL training loop, theoretically justifying optimization over a dynamic subset of tokens to reduce variance.

\subsection{Reinforcement Learning with Large Action Space}

Handling large or continuous action spaces is a longstanding challenge in RL research. Wolpertinger method \citep{wolpertinger} embeds actions in a continuous space to handle a large discrete action space. Yet, it assumes gradual change between actions. In domains such as RTS games or code generation, invalid-action masking \citep{invalid_action_making} is often used to manually filter out illegal moves based on rigid rules or syntax constraints. While effective, these methods rely on static, domain-specific rules that are difficult to scale to the open-ended semantic space of natural language. An exception is ASRE \citep{asre}, which automatically identifies sparse action to constrain its execution, but it is hard to handle the vast action space. Another line of research explores hierarchical RL (HRL) \citep{hiro,o3f}, which decomposes tasks into high-level planning and low-level execution. However, applying HRL to token-level generation often yields complex architectures that are difficult to tune. In contrast, \methodname~proposes a lightweight approach that leverages the LLM's prior knowledge to handle the large action space. 
\section{Experiment}
\label{sec:exp}

\begin{table*}[t]
\centering
\caption{Performance of various methods on diverse datasets. The values in the table represent the accuracy (\%) of the answers for different tasks. The reported results are evaluated by averaging 4 samples for each question.}
\label{tab:main_result}
\setlength{\tabcolsep}{8pt}
\resizebox{\linewidth}{!}{
\begin{tabular}{l|ccc|cc|c|c}
\toprule
\multirow{2}{*}{\textbf{Method}} & \multicolumn{3}{c}{\textbf{Math}} & \multicolumn{2}{c}{\textbf{Telecom}} & \textbf{Code} & \textbf{Average} \\
        \cmidrule(lr){2-4} \cmidrule(lr){5-6} \cmidrule(lr){7-7}
& Math-17k & AIME-24 & AIME-25 & Datacom & Wireless & OpenR1-Code \\
\midrule
No Training            & 25.2 & 16.7 & 16.7 & 55.65 & 50.87 & 40.89 & 34.34 \\  \midrule
GRPO            & 34.7 & 20.0 & \textbf{19.3} & 50.87 & 52.17 & 40.18 & 36.20 \\
\cellcolor{mygray} GRPO+\methodname            & \cellcolor{mygray}  \textbf{38.3} & \cellcolor{mygray} \textbf{23.3} & \cellcolor{mygray} 18.0 & \cellcolor{mygray} \textbf{51.30} & \cellcolor{mygray} \textbf{55.65} & \cellcolor{mygray} \textbf{40.92} & \cellcolor{mygray} \textbf{37.91} \\
\midrule
DAPO            & 36.4 & \textbf{20.7} & 17.3 & 54.43 & 49.57 & \textbf{39.87} & 36.38 \\
\cellcolor{mygray} DAPO+\methodname            & \cellcolor{mygray} \textbf{39.7} & \cellcolor{mygray} 19.3 & \cellcolor{mygray}  \textbf{20.7} & \cellcolor{mygray} \textbf{54.78} & \cellcolor{mygray} \textbf{50.43} & \cellcolor{mygray} 39.01 & \cellcolor{mygray} \textbf{37.32} \\
\bottomrule
\end{tabular}
}
\end{table*}

In this experiment, we conduct comprehensive experiments to evaluate the effectiveness of the proposed \methodname~method.
The main goal of our experiments is to answer the following key research questions:
(1) How does \methodname~compare with standard RL algorithms in terms of answer accuracy? (Sec. \ref{expsec:main_result})
(2) How does \methodname~improve the stability and quality of the RL training process? (Sec. \ref{expsec:performance_analysis})
(3) Is implicit token pruning superior to training an explicit selection module? (Sec. \ref{expsec:explict_policy})
and (4) How does the size of the promising token set affect performance? (Sec. \ref{expsec:ablation_study})
We begin by introducing the experimental setting.

\subsection{Experimental Setting}
\label{expsec:exp_setting}

\textbf{Datasets for evaluation.}
To verify the effectiveness of the proposed \methodname~method, we evaluate our approach on diverse benchmarks covering three representative domains:
\begin{itemize}[leftmargin=*]
\item \textbf{Mathematical reasoning:} We utilize Math-17k \citep{dapo}, AIME-24 and AIME-25 \citep{aime} to evaluate the model’s capacity for multi-step logical deduction. These tasks serve as a rigorous testbed for precision in token selection, where even minor stochastic deviations in early reasoning steps can lead to catastrophic error propagation.
\item \textbf{Code:} We leverage OpenR1-Code \citep{openr1} to evaluate proficiency in algorithmic synthesis. Unlike natural language, code generation necessitates navigating deterministic syntax and long-range functional dependencies, providing a benchmark for the model’s structural coherence.
\item \textbf{Telecom:} We also consider telecommunication tasks to verify broader application of \methodname~method, including Datacom and Wireless. We construct a dataset for each task comprising 12,000 question-answer pairs, synthesized from a diverse corpus of product documentation, technical solutions, and domain knowledge bases. The datasets encompass diverse question types, including single-choice, multiple-choice, and open-ended QA, covering fundamental principles, product concepts, terminology understanding, and multi-step reasoning tasks.
\end{itemize}
Appendix \ref{appsec:examples_of_dataset} shows examples of the input and output of these datasets.

\textbf{Baseline RL methods.} In our experiments, we select two representative RL methods as baselines: Group Relative Policy Optimization (\textbf{GRPO} \citep{grpo}), and Decoupled Clip and Dynamic Sampling Policy Optimization (\textbf{DAPO} \citep{dapo}). GRPO eliminates the critic model by estimating advantages through group-wise normalization of rewards. DAPO builds upon this framework by introducing dynamic sampling, which filters out prompt groups with identical rewards to ensure effective gradient updates, and decoupled clipping to mitigate entropy collapse.

\textbf{Implementation details.}
Experiments are implemented using the MindSpeed-RL framework \citep{mindspeed}. We employ Qwen3-4B and Qwen3-8B as base models.
Reward signals are provided by rule-based verifiers for Math, execution sandboxes for Code, and model-based evaluation (Deepseek-V3 \citep{deepseekv3}) for Telecom.
Unless stated otherwise, we use a promising set size of $K=4$ and report results at 200 training steps.
Computing resources includes clusters with 256 KUNPENG 920 CPUs and 8 Ascend 910B3 NPUs. Full hyperparameters are in Appendix \ref{appsec:hyper}. We refer the readers to Appendix \ref{appsec:more_details} for more details about the experimental setting.

\subsection{Main Results}
\label{expsec:main_result}

\textbf{Performance on mathematical tasks.}
Tab. \ref{tab:main_result} presents the comparative results on the Qwen3-8B backbone. \methodname~consistently outperforms baselines across all domains, with the most significant gains observed in mathematical reasoning. When applied to the GRPO baseline, \methodname~achieves a substantial improvement of 3.6\% on the Math-17k dataset (rising from 34.7\% to 38.3\%) and a 3.3\% gain on AIME-24.  While the improvements on AIME-25 are mixed, the overall trend suggests that pruning low-likelihood tokens effectively reduces the burden for maintaining the textual coherence on complex math problems.
This empirical evidence strongly validates our core hypothesis in Sec. \ref{sec:intro}: mathematical problem-solving involves searching through a vast combinatorial space of reasoning steps. Standard RL algorithms often waste sample efficiency exploring syntactically valid but logically irrelevant tokens (the ``long tail’'). By constraining the policy optimization to the \emph{promising} subspace, \methodname~effectively filters out this noise, allowing the agent to allocate its exploration budget exclusively to high-value logical decisions. 

\textbf{Generalization across domains.}
Beyond mathematics, \methodname~demonstrates robust generalization. In the Code domain (i.e., OpenR1-Code), it maintains a competitive edge, balancing the need for rigid syntax with algorithmic logic. In the Telecom domain, \methodname~improves over DAPO by \textbf{0.86\%} on Wireless tasks. These results confirm that identifying promising tokens is a domain-agnostic advantage that enhances sample efficiency regardless of the underlying knowledge domain.

\textbf{Training Efficiency.} As visualized in Fig. \ref{fig:training_curves}, \methodname~not only achieves higher asymptotic performance but also demonstrates superior sample efficiency. The reward curve rises sharply in the early stages compared to GRPO. This acceleration suggests that by removing the burden of maintaining syntactic coherence (which is already handled by the pre-trained prior), \methodname~enables the policy to “cut to the chase,” focusing immediately on optimizing reasoning trajectories.

\begin{figure}
\centering
\includegraphics[width=\linewidth]{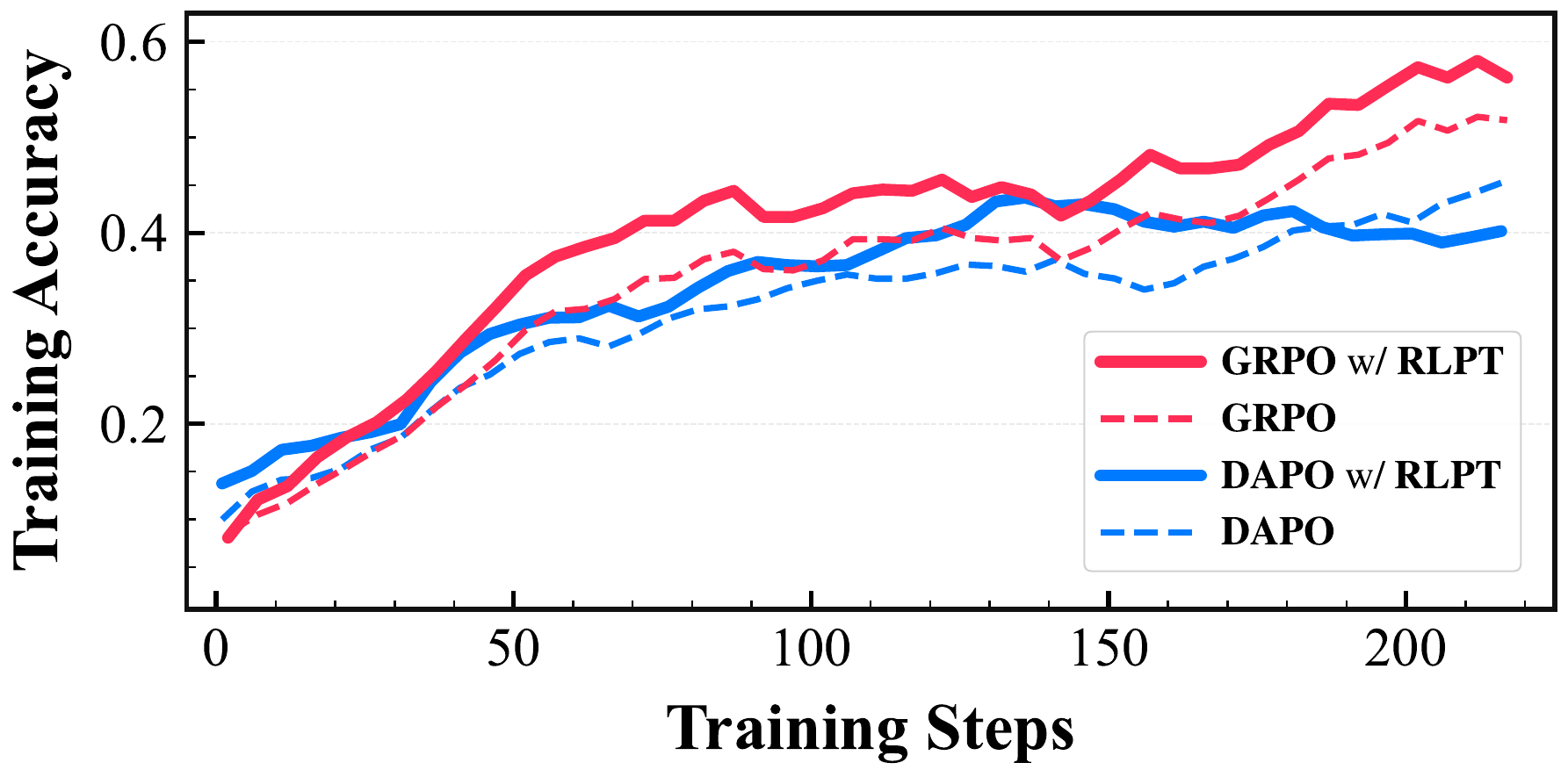}
\caption{Training curves on Math-17k dataset.}
\label{fig:training_curves}
\end{figure}

\textbf{Training with Qwen3-4B on Math-17k.}
To verify that the effectiveness of \textbf{RLPT} is not confined to a specific model capacity, we conduct an ablation study with the Qwen-4B model on the Math-17k dataset. The results, presented in Fig. \ref{fig:model_scale}, demonstrate that \methodname~consistently yields performance improvements on Qwen3-4B. As observed, \textbf{RLPT} provides a steady gain over the standard GRPO algorithm. Specifically, it achieves an absolute accuracy improvement of 1.3\% on the 4B model and a more pronounced 3.6\% on the 8B model. This consistent trend across different scales suggests that our proposed reinforcement learning strategy is robust and scales effectively with increased model capacity, further facilitating the discovery of successful reasoning trajectories.

\begin{figure}[htbp]
\centering
\includegraphics[width=\linewidth]{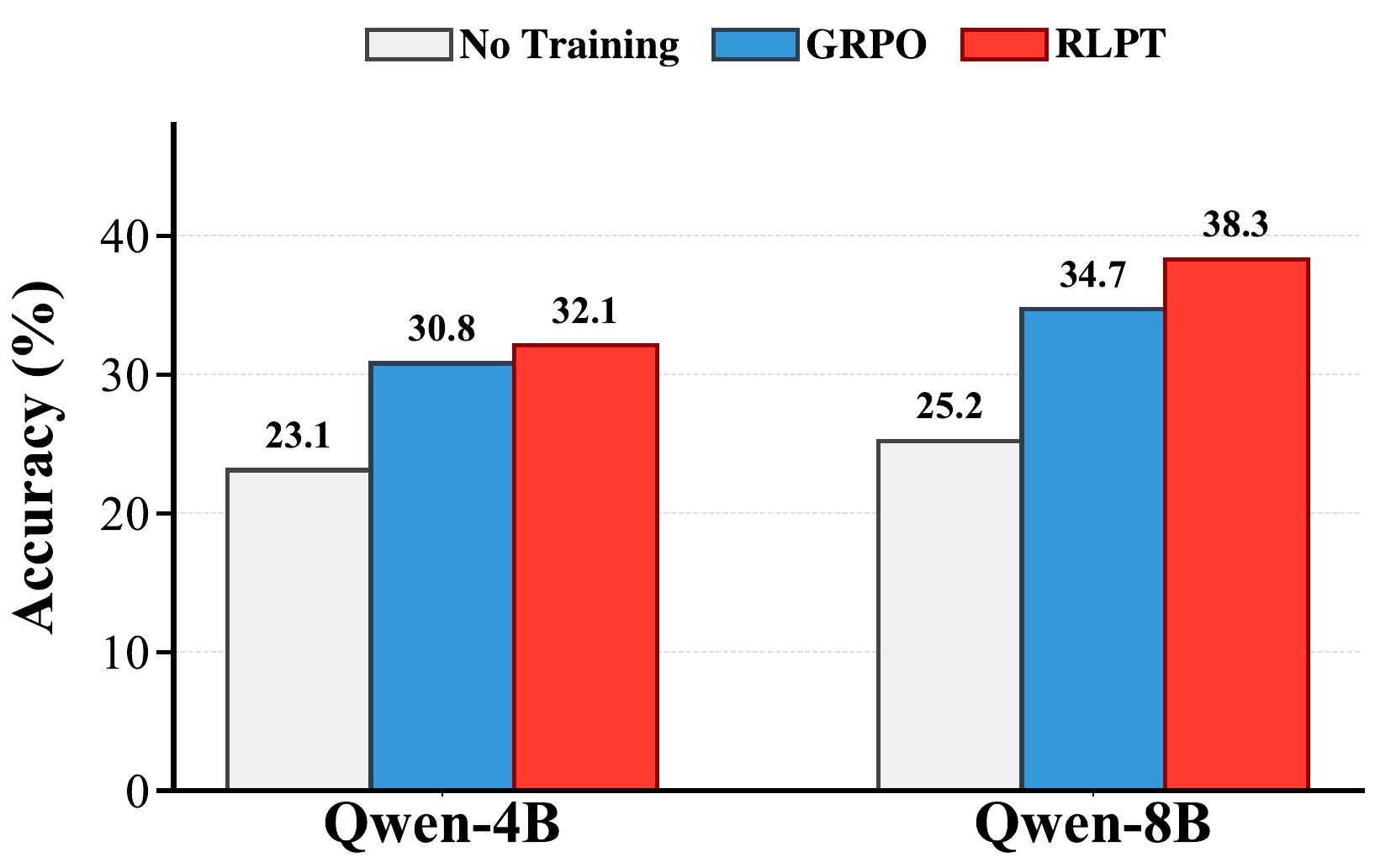}
\caption{Performance of \methodname~method on different sizes of models.}
\label{fig:model_scale}
\end{figure}

\subsection{Performance Analysis of \methodname~Method}
\label{expsec:performance_analysis}

To better understand why \methodname~improves training efficiency, we conduct a deeper analysis focusing on gradient variance and qualitative decision-making.

\textbf{Gradient norm during the training process}.
A key theoretical advantage of \methodname~is the reduction of variance in policy gradient estimation. By masking out the long tail of irrelevant tokens, the policy avoids assigning probability mass to invalid actions, thereby reducing noise in gradient updates. We visualize the gradient norm during the training of GRPO and GRPO+\methodname~in Fig. \ref{fig:grad_norm}. We observe that the gradient norm of \methodname~method is more stable and consistently lower in compared to the baseline. The baseline exhibits spikes, indicating unstable updates where the policy attempts to correct for large shifts in the probability of tail tokens. In contrast, \methodname~maintains a smoother optimization trajectory, confirming that constraining the action space acts as an effective regularizer.

\begin{figure}[t]
\centering
\includegraphics[width=\linewidth]{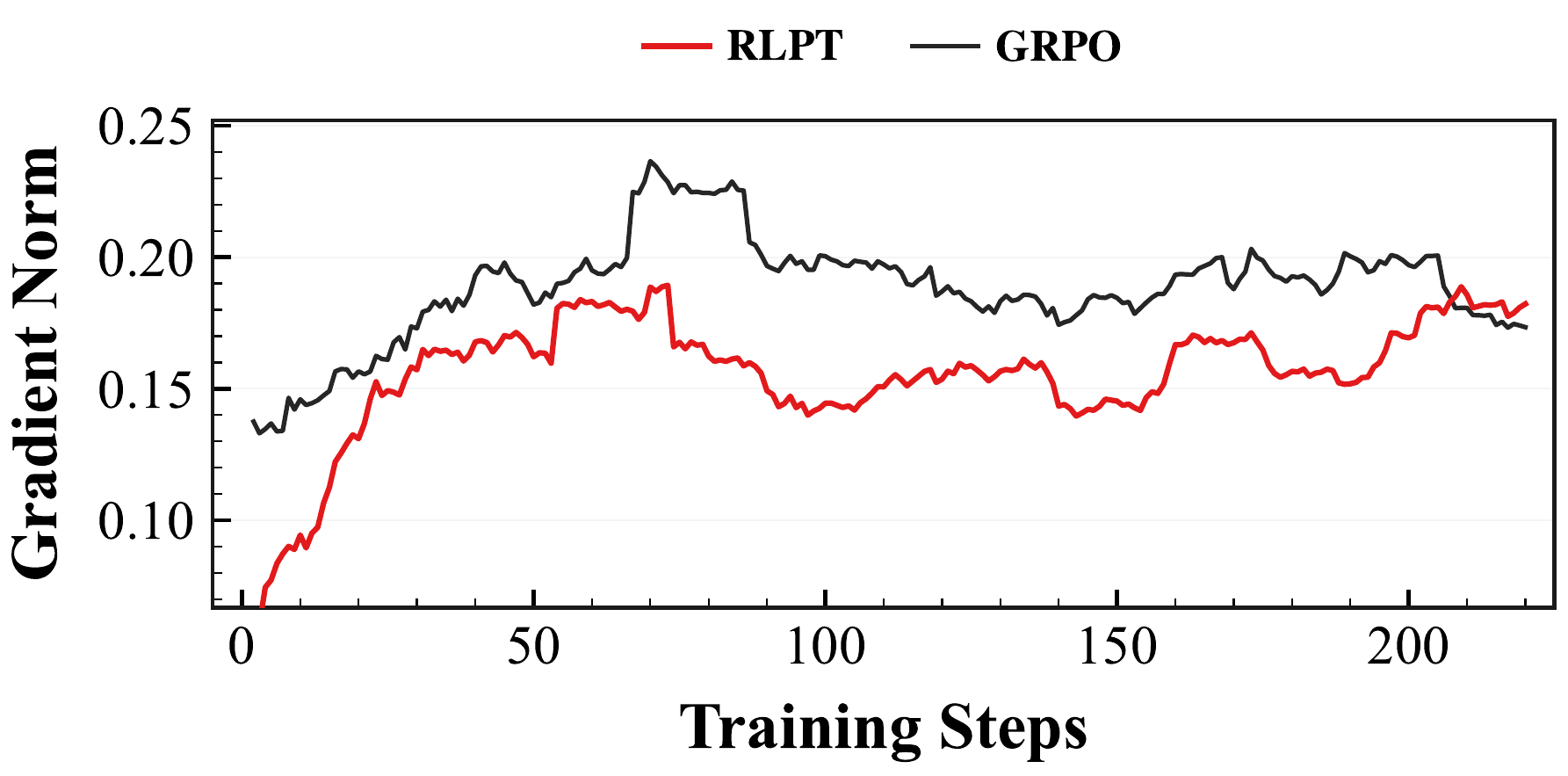}
\caption{Gradient norm curves of GRPO and GRPO+\methodname~during the training process.}
\label{fig:grad_norm}
\end{figure}

\textbf{Case study on decision making with promising token.}
To verify that the promising tokens correctly capture meaningful decision space, we examine specific inference steps from the GSM8K dataset \citep{gsm8k} using the trained model. Tab. \ref{tab:case_study} illustrates representative cases. Given a context requiring a reasoning step, the promising set contains distinct logical connectors, e.g., {so, since, but…}. Different selections could lead to diverse reasoning directions. Crucially, syntactically irrelevant words (e.g., “apple”, “dog”) are successfully filtered out. This observation confirms that \methodname~shifts the RL agent’s focus from \textit{what fits the sentence} to \textit{reasoning with the proper scenario}, effectively raising the abstraction level of the policy’s decisions.

\begin{table}[htbp]
\centering
\caption{Case study on the decision making over promising tokens.}
\begin{tabular}{p{4.5cm}|p{3cm}}
\toprule
\textbf{Contextual text} & \textbf{Promising tokens} \\   \toprule
A robe takes 2 bolts of blue fiber and half that much white fiber. How many bolts in total does it take? Answer: It takes 2/2=1 bolt of white fiber. & [Total, It, So, The, But, A, That, In] \\ \midrule
Every day, Wendi feeds each of her chickens three cups of mixed chicken feed, containing seeds, mealworms and vegetables to help keep them healthy. She gives the chickens their feed in three separate meals … 60 cups of feed per day. & [She, In, Since, Total, So, They, Thus, Each] \\
\bottomrule
\end{tabular}
\label{tab:case_study}
\end{table}

\subsection{Comparison with Explicit Selector Policy for Token Selection}
\label{expsec:explict_policy}

As the core idea of this work is to enable the policy to \textit{focus on decision making}, a natural alternative is to train an explicit, lightweight selector policy to filter tokens, rather than using our implicit probability-based masking. We implement this baseline by training a multilayer perceptron selector that takes the context embedding and candidate token embeddings as input to output the index for token selection. We pre-train this selector on ALPACA \citep{alpaca}, a general QA dataset, to learn general token relevance.

We compare this explicit selector method with \methodname~on Math-17k, using GRPO as the base optimization algorithm. The results are shown in Fig. \ref{fig:math17k_implicit}. We observe that the Explicit Selector method stops improving at a significantly lower score (\~25\%) than \methodname~(\~40\%). This failure can be attributed to the fact that the external selector is initialized from scratch, losing the general knowledge from LLM pre-training. In contrast, \methodname~leverages the \textit{intrinsic} semantic priors of the base model itself, ensuring that the promising set is always naturally aligned with the model’s capabilities.

\begin{figure}[h]
\centering
\includegraphics[width=\linewidth]{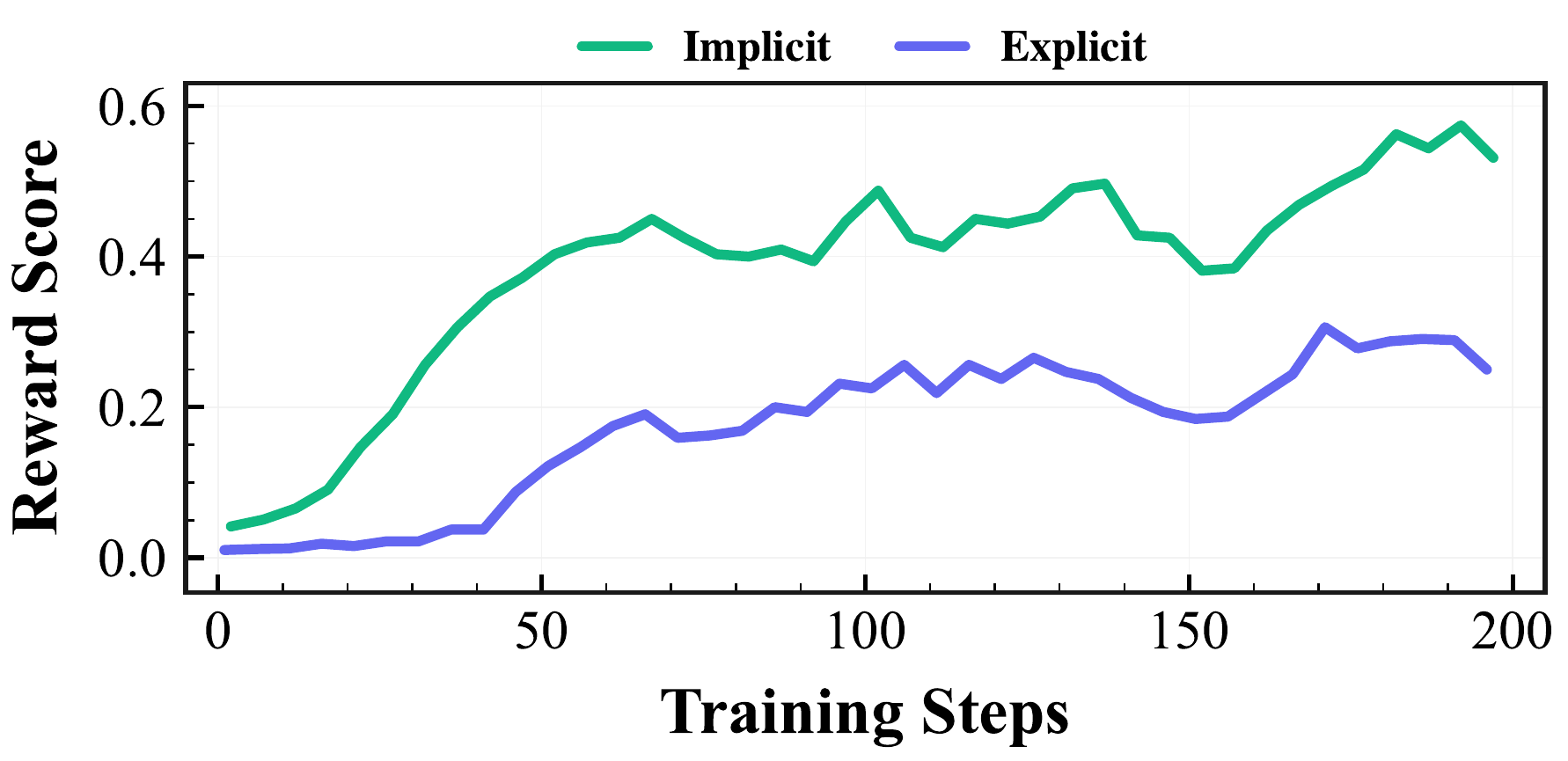}
\caption{Comparison of \methodname~training with explicit or implicit token selector.}
\label{fig:math17k_implicit}
\end{figure}

\subsection{Ablation Study}
\label{expsec:ablation_study}

\textbf{Training with different sizes of promising set.}
The hyperparameter $K$ controls the trade-off between the expressiveness of the action space and the efficiency of exploration. We evaluate $K \in {4, 8, 16, 32}$ on Math-17k, as shown in Fig. \ref{fig:ablation_k}.
Interestingly, we find that a compact set of $K=4$ yields the best asymptotic performance.
Increasing $K$ to 32 results in a noticeable performance drop, regressing towards the baseline. This supports our premise that the “true” reasoning path typically lies within the top few candidates. Expanding $K$ unnecessarily reintroduces the noise of the large vocabulary, diluting the RL signal. Therefore, a tight constraint ($K=4 \sim 8$) strikes the optimal balance: it is large enough to include the correct reasoning step (i.e., high recall) but small enough to enforce focused decision-making (i.e., high signal-to-noise ratio).

\begin{figure}[t]
\centering
\includegraphics[width=\linewidth]{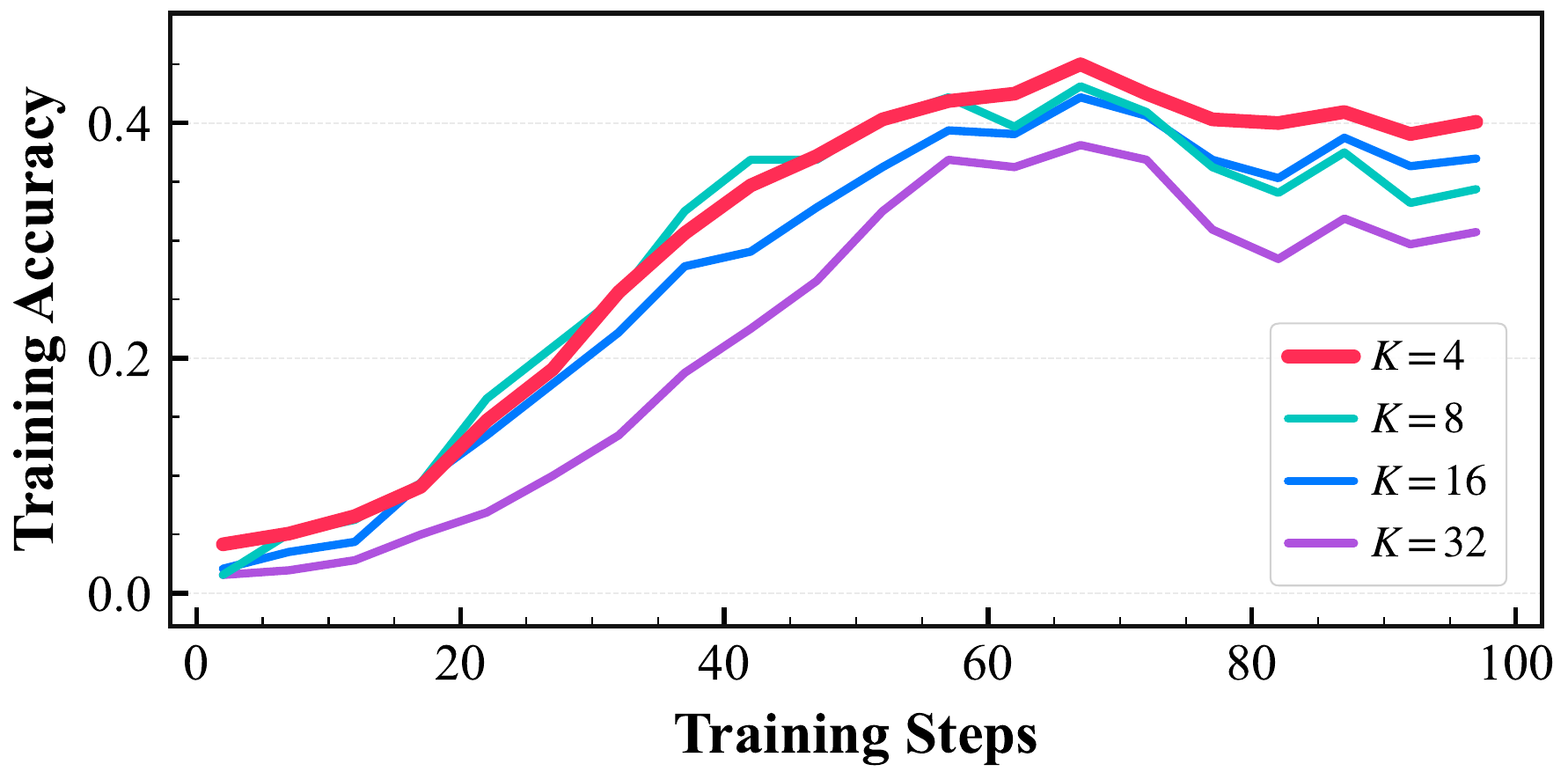}
\caption{Ablation study on the size of promising set (K).}
\label{fig:ablation_k}
\end{figure}
\section{Conclusion \& Future Work}
\label{sec:conclusion}
In this work, we proposed \methodname, a RL framework that decouples strategic decision-making from basic token generation. By restricting the RL policy to a dynamic set of promising tokens identified by the base model's semantic priors, we effectively transform the high-dimensional vocabulary space into a manageable decision space. Our theoretical analysis and empirical evaluations on mathematical reasoning and coding tasks demonstrate that RL with promising tokens clearly reduces gradient variance and improves sample efficiency.

Despite the improvements in training stability and performance, there are still some limitations. Currently, \methodname~relies on a simple top-$k$ ranking to define the promising token set. While effective, fixed-size subsets may be sub-optimal. Future work could explore more adaptive thresholding methods such as top-$p$ sampling or min-$p$ sampling \citep{minp}, which might better capture the dynamic nature of token uncertainty across different reasoning steps. Restricting the action space carries the inherent risk of excluding critical decision tokens from the candidate set. While our empirical analysis shows that successful trajectories are over 99\% covered by the top-32 tokens, the missing $<1\%$ could be necessary for complex reasoning or tasks with complicated symbols. We hope that future research will explore these intriguing questions and contribute to developing LLMs that could learn more effectively.

\bibliography{reference}
\bibliographystyle{icml2026}

\newpage
\appendix

\renewcommand{\thepart}{}
\renewcommand{\partname}{}

\onecolumn
\begin{center}
	\part{\huge{\textbf{Appendix}}} %
\end{center}
\parttoc %

\section{Discussion}
\label{appsec:discussion}

\subsection{The Off-policy Issue for LLM Optimization with Standard RL}

we discuss a theoretical inconsistency prevalent in standard RL baselines regarding the mismatch between the behavior policy (used for rollout) and the target policy (used for optimization).Standard RL approaches for LLMs (e.g., PPO) typically employ decoding strategies such as Top-$k$ or Nucleus (Top-$p$) sampling during the data collection (rollout) phase to prevent the generation of incoherent text from the long tail. Let $\pi_\theta$ denote the parameterized policy over the full vocabulary $\mathcal{V}$. The actual behavior policy $\pi_{b}$ during rollout is a truncated version:
$$\pi_{b}(a_t|s_t) = \text{Truncate}(\pi_\theta(a_t|s_t); k, p),$$
where the support of $\pi_{b}$ is strictly smaller than $\mathcal{V}$. However, during the optimization phase, standard baselines typically calculate policy gradients (and KL-divergence penalties) based on the original distribution $\pi_\theta$ over the full vocabulary. 

This discrepancy turns standard baselines into an unacknowledged off-policy setting with biased estimators. In contrast, \methodname~resolves this inconsistency by integrating the token masking directly into the policy definition. By optimizing over the same masked subspace used for generation, we ensure that the behavior policy and the optimization policy are mathematically consistent, adhering to the on-policy assumption.

\section{Omitted Proof}

\subsection{Proof of Proposition \ref{prop:variance}}
\label{appsec:proof}

\textbf{Proposition \ref{prop:variance} (Variance Reduction)}
\textit{Assuming the advantage $A_t$ is bounded, optimizing the policy over the constrained space $\mathcal{P}_t$ strictly reduces the variance of the gradient estimator associated with the tail tokens $\mathcal{T}$, compared to optimization over the full vocabulary $\mathcal{V}$.}

\begin{proof}
First, we calculate the variance of the gradient component for a single token $i$. Since $a_t \sim \pi(\cdot|s_t)$, the term $\mathbb{I}(a_t=i)$ is a Bernoulli variable with parameter $\pi_i = \pi(i|s_t)$. The variance of $\hat{g}_i$ is:
\begin{equation}
    \text{Var}(\hat{g}_i) = \text{Var}\left( (\mathbb{I}(a_t=i) - \pi_i) A_t \right) = \pi_i (1 - \pi_i) A_t^2.
\end{equation}

For the standard optimization over the full vocabulary $\mathcal{V}$, the total variance sums over all tokens:
\begin{equation}
    \mathbb{V}(\hat{g}_{\text{full}}) = \sum_{i \in \mathcal{V}} \pi_i (1 - \pi_i) A_t^2 = \sum_{i \in \mathcal{P}_t} \pi_i (1 - \pi_i) A_t^2 + \sum_{k \in \mathcal{T}_t} \pi_k (1 - \pi_k) A_t^2.
\end{equation}

In \methodname, the policy is renormalized to $\tilde{\pi}$ over $\mathcal{P}_t$, and the logits for tail tokens $k \in \mathcal{T}_t$ are masked (i.e., gradients are deterministically zero). Thus, $\text{Var}(\hat{g}_{\text{mask}, k}) = 0$ for all $k \in \mathcal{T}_t$. The total variance becomes:
\begin{equation}
    \mathbb{V}(\hat{g}_{\text{mask}}) = \sum_{i \in \mathcal{P}_t} \tilde{\pi}_i (1 - \tilde{\pi}_i) A_t^2.
\end{equation}
Since the tail probability mass is negligible (as empirically verified), we have $\tilde{\pi}_i \approx \pi_i$ for $i \in \mathcal{P}_t$. The reduction in variance is dominated by the removal of the tail sum:
\begin{equation}
    \Delta \mathbb{V} = \mathbb{V}(\hat{g}_{\text{full}}) - \mathbb{V}(\hat{g}_{\text{mask}}) \approx \sum_{k \in \mathcal{T}_t} \pi_k (1 - \pi_k) A_t^2.
\end{equation}
Since $\pi_k > 0$ for tokens in the tail (due to the Softmax properties) and $A_t^2 \ge 0$, the term $\sum_{k \in \mathcal{T}_t} \pi_k (1 - \pi_k) A_t^2$ represents a strictly positive quantity.
This explicitly shows that standard RL optimization suffers from noise injection caused by the gradients of thousands of irrelevant tail tokens, which \methodname~effectively eliminates.
\end{proof}

\section{More Details about Experiment Settings}
\label{appsec:more_details}

\subsection{Examples of the Datasets in Our Experiments}
\label{appsec:examples_of_dataset}

Tab. \ref{appsec:examples_of_dataset} presents the examples from the datasets used for training and evaluation in our experiments.

\begin{table}[h]
    \centering
    \renewcommand{\arraystretch}{1.5} 
    \begin{tabular}{c|p{7.5cm}|p{3cm}}
    \toprule
    \textbf{Dataset} & \centering \textbf{Example Question} & \textbf{Target Answer} \\
    \midrule
    
    Math17k & ``In triangle $ABC$, $\sin \angle A = \frac{4}{5}$ and $\angle A < 90^\circ$. Let $D$ be a point outside triangle $ABC$ such that $\angle BAD = \angle DAC$ and $\angle BDC = 90^\circ$. Suppose that $AD = 1$ and that $\frac{BD}{CD} = \frac{3}{2}$. If $AB + AC$ can be expressed in the form $\frac{a\sqrt{b}}{c} \dots$'' & ``The answer is 34.'' \\ \midrule
    
    AIME24 & ``Every morning Aya goes for a $9$-km-long walk and stops at a coffee shop. When she walks at $s$ km/h, the walk takes 4 hours, including $t$ minutes in the shop. When she walks $s+2$ km/h, it takes 2h 24m, including $t$ minutes. If she walks at $s+\frac{1}{2}$ km/h, find the total minutes \dots'' & ``The answer is 204.'' \\ \midrule

    AIME25 & ``In $\triangle ABC$, $D, E \in \overline{AB}$ and $F, G \in \overline{AC}$. Given $AD=4, DE=16, EB=8$ and $AF=13, FG=52, GC=26$. Let $M$ be the reflection of $D$ through $F$, and $N$ be the reflection of $G$ through $E$. If Area($DEGF$) = 288, find the area of heptagon $AFNBCEM$.'' & ``The answer is 588.'' \\ \midrule

    OpenR1-Code & ``Polycarp has $n$ different binary words. He wants to reverse minimal number of words so that the final set can be arranged in a game sequence where each word starts with the previous word's last character \dots'' & ``\texttt{if zo > oz: (zo-oz)//2 \dots else: (oz-zo)//2 \dots}'' \\ 
    
    \bottomrule
    \end{tabular}
    \caption{Examples from the mathematical reasoning datasets.}
    \label{tab:math_examples}
\end{table}

\begin{table}[h]
    \centering
    \begin{tabular}{c|p{5.2cm}|p{4.5cm}}
   \toprule
   \textbf{Domain} & \centering \textbf{Example question} & \textbf{Target answer} \\
   \midrule
   Wireless-training & ``What should I do if the DSP GTPPATH command output shows that the GTP path is in the DEETECT state when the alarm is generated?" & "\dots1.First, check whether the peer GSN address specified in the alarm information is valid\dots;2.Next, execute the PING command to check if the link is normal\dots;3.Confirm whether the peer GSN can respond to ECHO messages\dots'' \\ \midrule
   Wireless-testing & ``During the deployment and operation of the LMT, the LMT is forcibly connected in HTTPS or WSS mode to ensure secure connection. In this mode, digital certificates are required for authentication. In addition, \dots In such a deployment scenario, if the MAE of a carrier is set to HTTP login mode and the LMT is set to forcible HTTPS connection mode, what will happen? A. \dots B. If the LMT connection mode is set to Force HTTPS, MAE proxy login fails to access the LMT due to protocol mismatch. The connection cannot be established even if the OM channel is normal. C. \dots D. \dots" & ``The answer is X.'' \\ \midrule
    Datacom-training & ``When configuring the Segment VXLAN feature, how can you enable EVPN as the VXLAN control plane on Transit Leaf1 and Transit Leaf2, and configure BGP EVPN peer relationships?" & "\dots1.Enter the BGP view or BGP multi-instance view\dots;2.Enter the BGP-EVPN address family view\dots;3.Configure the split group for BGP EVPN peers (groups)\dots;4.Enable the function to mark routes received from BGP EVPN peers as re-originated\dots'' \\ \midrule
    Datacom-testing & ``\dots The device supports creating subinterfaces on Layer 2 Ethernet and Layer 2 Eth-Trunk interfaces for VLAN termination to achieve inter-VLAN forwarding. However, the USG9500 series devices do not support creating subinterfaces on these two types of interfaces." & ``The answer is error'' \\ 
   \bottomrule
\end{tabular}
    \caption{Examples from the datasets used in our experiments.}
    \label{tab:example_of_dataset}
\end{table}

\subsection{Prompts Used in Our Experiments}
\label{appsec:prompts}

Tab. \ref{tab:prompts} presents the prompts used in our exeriments.

\begin{table}[h]
    \centering
    \caption{Prompts for Experiments.}
    \begin{tabular}{c|p{7.5cm}}
    \toprule
    \textbf{Task Name} & \textbf{Prompt}  \\   \toprule
    Math & Let's think step by step. Output the final answer within \verb|\boxed{}|, e.g., \verb|\boxed{5}|. \\ \midrule
    Code & Solve the following coding problem using the programming language python: \verb|{problems}| \\ \midrule
    Telecom & For the following multiple-choice question, there is only one correct answer. Please analyze the question and the options, and place the correct option ID within \verb|\boxed{}|, e.g., \verb|\boxed{A}|. \\
    \bottomrule
    \end{tabular}
    \label{tab:prompts}
\end{table}

\subsection{Hyperparameters}
\label{appsec:hyper}

The hyper-parameters for implementing \methodname~and experiments are presented in Tab. \ref{tab:hyper-parameters}. When implementing baseline methods, we use the same hyper-parameters as \methodname.

\begin{table}[htbp]
    \centering
    \caption{Hyper-parameters for training \methodname~and baselines.}
    \begin{tabular}{l|l}
    \toprule
    \textbf{Hyper-parameters} & \textbf{Value}  \\   \toprule
    Batch size & $8$ \\
    Learning rate & $10^{-6}$ \\
    Learning rate decay style & $cosine$ \\
    Train iteration & $200$ \\
    Sequence length & $4096$ \\
    RL $\gamma$ & $1$ \\
    RL $\lambda$ & $0.95$ \\
    Mini batch size & $4$ \\
    RL clip ratio & $0.2$ \\
    Promising set size & $4$ \\
    Entropy coefficient & $0.0$ \\
    KL coefficient & $0.01$ \\
    Rollout count & $8$ \\
    \bottomrule
    \end{tabular}
    \label{tab:hyper-parameters}
\end{table}

\end{document}